\documentclass[journal]{IEEEtran}
\usepackage{cite}

\ifCLASSINFOpdf
\else
\fi
\usepackage{url} 
\usepackage{graphicx}
\usepackage{subcaption}
\usepackage{booktabs}
\usepackage{multirow}
\usepackage{pifont}
\newcommand{\cmark}{\ding{51}}
\newcommand{\xmark}{\ding{55}}
\usepackage{xcolor}
\usepackage{amsmath, amssymb, amsthm}

\newtheorem{theorem}{Theorem}
\newtheorem{proposition}{Proposition}

\newtheorem{lemma}{Lemma}
\newtheorem{remark}{Remark}

\begin{document}
%
\title{OmniLoc: A Geometry-Aware Foundation Model for Anchor-Free UE Localization Across Diverse Indoor Environments}

\author{Lei~Chu,~\IEEEmembership{Senior Member,~IEEE,}
        ~Yuning ~Zhang, ~Omer Gokalp ~Serbetci, ~Anushka ~Katiyar,  ~Bassel Abou Ali ~Modad, 
        and~Andreas F. ~Molisch,~\IEEEmembership{~Fellow,~IEEE}
}

\markboth{Journal of \LaTeX\ Class Files,~Vol.~14, No.~8, August~2015}%
{Shell \MakeLowercase{\textit{et al.}}: Bare Demo of IEEEtran.cls for IEEE Journals}

\maketitle

\begin{abstract}
Indoor localization from wireless measurements remains challenging in large-scale deployments due to substantial variation in building geometry, the set of detectable access points (APs), and the heterogeneity of received signals. Existing learning-based methods often perform well only in limited settings and degrade under environmental shifts, making robust anchor-free localization across diverse indoor environments notoriously difficult. In this paper, we present OmniLoc, an environment-interactive foundation model for anchor-free user equipment localization across diverse indoor environments. To the best of our knowledge, OmniLoc is the first foundation-model-based approach built directly on wireless measurements for this task. OmniLoc is built on three key designs. First, a unified input tokenization module converts heterogeneous wireless measurements into a common representation that is more amenable to learning. Second, a geometry-aware Transformer performs AP-aware feature extraction by emphasizing dominant APs while aggregating complementary evidence from supporting APs. Third, a geometry-aware location estimation module conditions regression on geometric embeddings to produce geometrically consistent location predictions. We evaluate OmniLoc on both a large-scale in-house dataset and a public benchmark dataset. Results show that OmniLoc significantly outperforms existing methods, consistently improves existing backbones when its design components are integrated, and demonstrates strong generalization in cross-environment evaluations.
\end{abstract}


\IEEEpeerreviewmaketitle

\section{Introduction}

Indoor user equipment (UE) localization has become an essential capability for next-generation wireless systems \cite{zekavat2019handbook}, \cite[Chap. 29]{molisch2022wireless}, \cite{witrisal2016high,zafari2019survey},  enabling a wide range of applications such as context-aware services \cite{zhang2020context}, asset tracking \cite{mayer2023self}, emergency response \cite{ferreira2017localization}, and intelligent building management \cite{elbes2022platform}. ``Classical'' methods such as those based on time of arrival (which includes the popular Global Positioning System GPS and cellular 911 localization), direction of arrival, and proximity sensing, have made substantial progress, but still face fundamental challenges in particular in indoor environments. For these reasons, machine-learning (ML) based techniques have gained popularity in particular for indoor localization, see the surveys  \cite{burghal2020comprehensive, singh2021machine, tiku2023overview,kerdjidj2024uncovering} and references therein. However, this literature also shows the challenges in this approach: real deployments exhibit strong geometric diversity across environments, large variation in the number of detectable APs, and highly heterogeneous received signal strengths caused by different propagation conditions \cite{aditya2018survey}, blockage \cite{chu2024exploiting}, and device-environment interactions \cite{torsoli2023blockage}. 
These factors make it challenging for existing methods to achieve robust and transferable performance, especially in anchor-free settings, where localization relies directly on wireless measurements without explicit infrastructure calibration\cite{xie2015precise, tong2026error}.

Meanwhile, recent advances in  foundation models \cite{xu2024large,ott2024radio,pan2025large,liu2025wifo,si2025cross,aboulfotouh20256g, cheraghinia2025foundation,xiao2026wireless}, have demonstrated high potential for learning transferable representations from large-scale and diverse data. For instance, the pretrained transformer model \cite{ott2024radio} is capable of capturing spatiotemporal patterns and environmental propagation characteristics, thereby enabling highly accurate 5G localization. Using the simulated DeepMIMO dataset \cite{alkhateeb2019deepmimo}, \cite{pan2025large} demonstrated that a transformer-based, self-supervised foundation model tailored for wireless localization can achieve superior accuracy with minimal labeled data, while also exhibiting strong robustness to previously unseen base station configurations. 
However, their direct application to wireless indoor localization from raw measurements remains largely underexplored. Unlike conventional learning tasks, indoor localization demands not only effective modeling of heterogeneous wireless signals \cite{zhang2026wiloc}, 
but also the preservation of the underlying geometric relationships embedded within these observations. This calls for a new framework that jointly captures measurement heterogeneity 
and geometric consistency while generalizing across diverse indoor environments. The core contributions of this work includes: 
\begin{itemize}
    \item We propose OmniLoc, an environment-interactive foundation model for anchor-free user equipment localization across diverse indoor environments. To the best of our knowledge, OmniLoc is the first foundation-model-based approach built directly on wireless measurements, explicitly addressing core challenges such as geometric diversity, variable numbers of detectable APs, and heterogeneous received signal strengths. 
    \item We introduce three key designs in OmniLoc: 1) a unified input tokenization module that converts heterogeneous wireless measurements into a \textit{sentence-like} representation more amenable to learning; 2) a geometry-aware Transformer that performs AP-aware feature extraction by emphasizing dominant APs while aggregating complementary evidence from others; and 3) a geometry-aware location estimation module that conditions regression on geometric embeddings to produce geometrically consistent UE location estimates. 
    \item We conduct extensive case studies on both our collected dataset and a public benchmark dataset. The results show that: 1) OmniLoc achieves significant improvements over existing methods in large-scale and diverse indoor environments; 2) the proposed design components are flexible, readily integrate with existing methods, and consistently improve their performance; and 3) OmniLoc demonstrates strong generalization in cross-environment evaluations. 
\end{itemize}

The remainder of this paper is organized as follows. Section II reviews the related work. Section III introduces the problem formulation and the proposed method. Section IV presents the case studies and corresponding analysis. Finally, Section V concludes the paper. To improve readability, we defer the key theoretical analysis and proofs to the appendix.

\section{Related works}

\subsection{Indoor Localization without Anchors}
Indoor localization has emerged as a critical enabler for location-aware applications \cite{zafari2019survey}.  Traditional anchor-based systems rely on reference nodes with known location — such as Wi-Fi access points (APs) \cite{abbas2019wideep}, Bluetooth beacons \cite{danics2017model}, or Ultra-Wideband (UWB) transceivers \cite{gezici2005localization} — whose positions must be pre-surveyed and maintained, resulting in high deployment costs and labor-intensive calibration; furthermore, in numerous situations the locations of those nodes might not be easily known (e.g., WiFi APs in private apartments) or kept confidential for security reasons \cite{zhang2026wiloc}. Anchor-free localization addresses these limitations by eliminating dependency on infrastructure with known locations, instead inferring spatial relationships directly from inter-node measurements or environmental signals \cite{priyantha2003anchor, youssef2005accurate, shioda2013anchor, tan2013connectivity, zhou2025cooperative}. 

A substantial body of research has explored anchor-free indoor localization using diverse sensing modalities and algorithmic frameworks. Early work combined relative range measurements — obtained via Received Signal Strength Indicator (RSSI) \cite{konings2019springloc} or time-of-flight (ToF) \cite{wendeberg2011anchor} — with multidimensional scaling (MDS) \cite{pei2009anchor, saeed2019state} to reconstruct node configurations without fixed references. Later studies incorporated Inertial Measurement Unit (IMU) data through Pedestrian Dead Reckoning (PDR) \cite{kang2014smartpdr}, fused with opportunistic signals such as geomagnetic fingerprints or barometric pressure to mitigate drift. More recently, deep learning approaches — including Long Short-Term Memory (LSTM) networks \cite{lehyeh2025ubigtloc} for sequential motion modeling — and graph-based Simultaneous Localization and Mapping (SLAM) optimization using loop-closure constraints have further improved positioning accuracy and robustness \cite{chen2021anchor}. Collectively, these works reflect the growing maturity of anchor-free methods, converging toward infrastructure-independent solutions that combine signal processing, probabilistic inference, and data-driven learning. 

\subsection{Wireless Signal Modalities for Indoor Localization}
 
Indoor localization has been explored with a wide range of wireless modalities \cite{wu2012csi, chu2014robust, xiao2016survey, zafari2019survey}, each offering a different trade-off among accuracy, coverage, deployment cost, and robustness. Early systems primarily relied on RSSI and/or signal-to-noise ratio (SNR) or Ssignal-to-interference-plus-noise ratio (SINR)  measurements derived from Wi-Fi, Bluetooth, ZigBee, and cellular signals, since they are readily available in existing infrastructure and require minimal additional deployment. This makes them attractive for scalable fingerprinting and lateration-based systems. However, RSSI is highly sensitive to multipath, temporal variation, device heterogeneity, and environmental dynamics, which often limits localization accuracy. To overcome these limitations, later work \cite{thayaparan2008micro, jung2011tdoa, wu2012csi, yassin2016recent, wen2018joint} exploited richer physical-layer measurements, including channel state information (CSI), time of arrival (ToA), time difference of arrival (TDoA), angle of arrival (AoA), phase, and Doppler features, which provide richer spatial information for more accurate modeling of indoor propagation.  

Despite this progress, UE localization in large-scale, heterogeneous indoor environments remains difficult, largely because effective input representations are hard to design \cite{luo2014piloc,zhang2017large,lymberopoulos2017microsoft}. Prior work often relies on Fourier-domain representations, such as power delay profiles \cite{vinogradov2017physical}, angular power delay profiles \cite{wu2021learning}, and Doppler features \cite{thayaparan2008micro}, or on channel-charting techniques \cite{studer2018channel,taner2025channel} to capture fine-grained signal structure \cite{chintalapudi2010indoor,wu2014smartphones,lymberopoulos2017microsoft,wang2021secure,hu2023wisdom}. However, in large commercial or campus Wi-Fi deployments, CSI phase measurements are often unstable and noisy due to the limited number of transmit and receive antennas, making them unreliable and prone to substantial error \cite{xie2015precise}. Motivated by these limitations, we seek a unified representation that is both robust and expressive. Inspired by interpretable sentence embedding methods \cite{lin2017structured}, our design is tailored to the structure of wireless measurements, making it readily extensible and well suited to large-scale indoor localization.

\subsection{Advanced ML Models for Indoor Localization}

Machine learning has become an important approach to indoor localization \cite{wang2016csi,wang2018deep,burghal2020comprehensive}. Early systems largely relied on classical models, such as extreme learning machines \cite{zhang2017large}, k-nearest neighbors \cite{xie2016improved}, and support vector machines \cite{tran2008localization}. More recent work has shifted toward deep neural networks \cite{wang2018deep}, which learn directly from RSSI, CSI, and other wireless measurements. In particular, methods based on convolutional neural networks (CNNs), LSTMs, and autoencoders outperform classical pipelines by capturing richer spatial and temporal structures in wireless fingerprints. 

More recently, researchers have explored more expressive architectures, including Transformers \cite{masrur2025transforming,wu2025transaoa,pan2025large}, self-supervised pretraining \cite{paulavivcius2022temporal,ott2024radio,salihu2024self}, and geometry-aware representation learning \cite{salihu2024self,chu2024exploiting}. These approaches better capture complex dependencies among wireless observations and improve robustness under challenging conditions such as non-line-of-sight propagation and sparse supervision. At the same time, emerging studies have begun to investigate foundation-model-style solutions for wireless localization. However, existing methods still depend heavily on synthetic data or single-environment training, and they do not fully address anchor-free localization across diverse indoor geometries and heterogeneous measurements.

\section{Anchor-Free UE Localization: System Overview, Challenges, and Problem Formulation}

\label{section:system}

\subsection{Localization System}
\label{section:system1}


To investigate anchor-free localization in large-scale Wi-Fi deployments, we conducted an extensive measurement campaign \cite{zhang2026wiloc} at the University of Southern California (USC) University Park Campus. The campaign spans 16 buildings and 180 corridors, covering a total length of 4,560 m and detecting 3,293 access points (APs). UE locations are sampled at a default spatial resolution of 1.5 mm. On the UE side, a Universal Software Radio Peripheral (USRP) is used to digitize and record signals across the full 100 MHz bandwidth of the 2.4 GHz industrial, scientific, and medical (ISM) band. The APs operate on different channels, and beacon reception is performed using a single antenna, with no transmission from the UE; consequently, multi-antenna effects and AP-side beamforming are not considered. Each measurement is further annotated with 14 metadata attributes to capture relevant system parameters. Additional details can be found in our recent dataset paper \cite{zhang2026wiloc}, and the dataset is publicly available on Wides website.
In this work, we use the CSI measured over 52 subcarriers, denoted by $\mathbf{C}_t \in \mathbb{R}^{52 \times K}$, where $K$ is the number of APs observed at UE location $t$ and may vary across locations. We also use the RSSI vector $\mathbf{r}_t \in \mathbb{R}^{K}$, the signal-to-interference-plus-noise ratio (SINR) vector $\mathbf{s}_t \in \mathbb{R}^{K}$, and the UE location at timestamp $t$, together with noise and interference measurements.
To encode AP availability, we introduce AP presence indicators at both local and global scales. Specifically, $\mathbf{f}_t \in \mathbb{R}^{N_1}$  denotes the AP presence vector within the current building, whereas $\mathbf{g}_t \in \mathbb{R}^{N_2}$ represents AP presence over the union of APs across all buildings. Here, $N_1$ denotes \textit{ the maximum AP dimension used within a building, corresponding to the total number of distinct APs observed across its floors}. These indicators explicitly specify whether each AP is visible at the current UE location and are therefore particularly useful for handling variable AP availability. For instance, two nearby locations may observe 12 and 20 APs, respectively, with partial overlap between their visible AP sets. To obtain a consistent representation, we pad each sample to the fixed dimension $N_1$ and assign zeros to APs that are not observed at the current location. In this way, $\mathbf{f}_t$ and $\mathbf{g}_t$ provide \textit{aligned local and global AP visibility cues while preserving a fixed AP ordering}.

\subsection{Challenges}

\begin{figure}[!t]
    \centering
    \begin{subfigure}{0.241\textwidth}
        \centering
        \includegraphics[width=\linewidth]{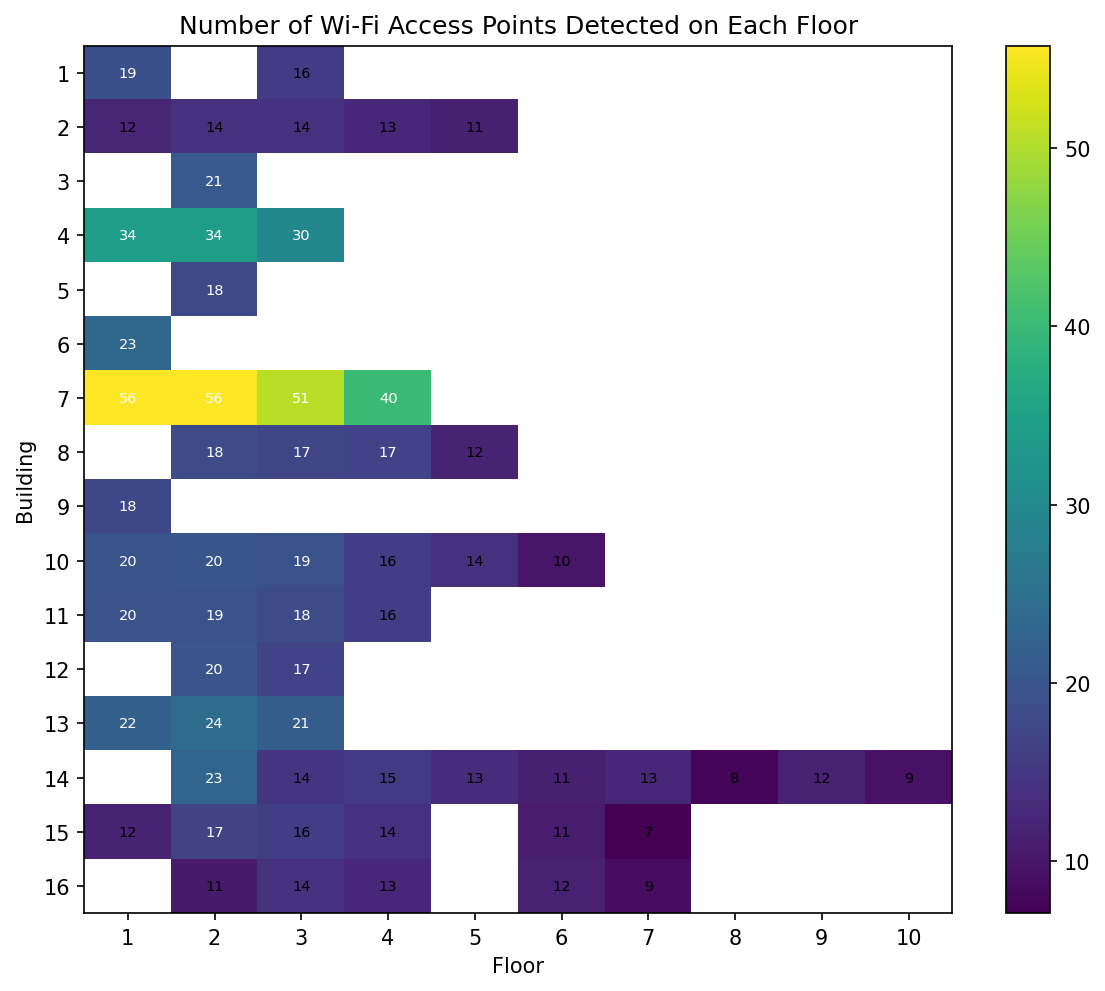}
        \caption{AP Numbers}
        \label{fig:sub1}
    \end{subfigure}
    \begin{subfigure}{0.241\textwidth}
        \centering
        \includegraphics[width=\linewidth]{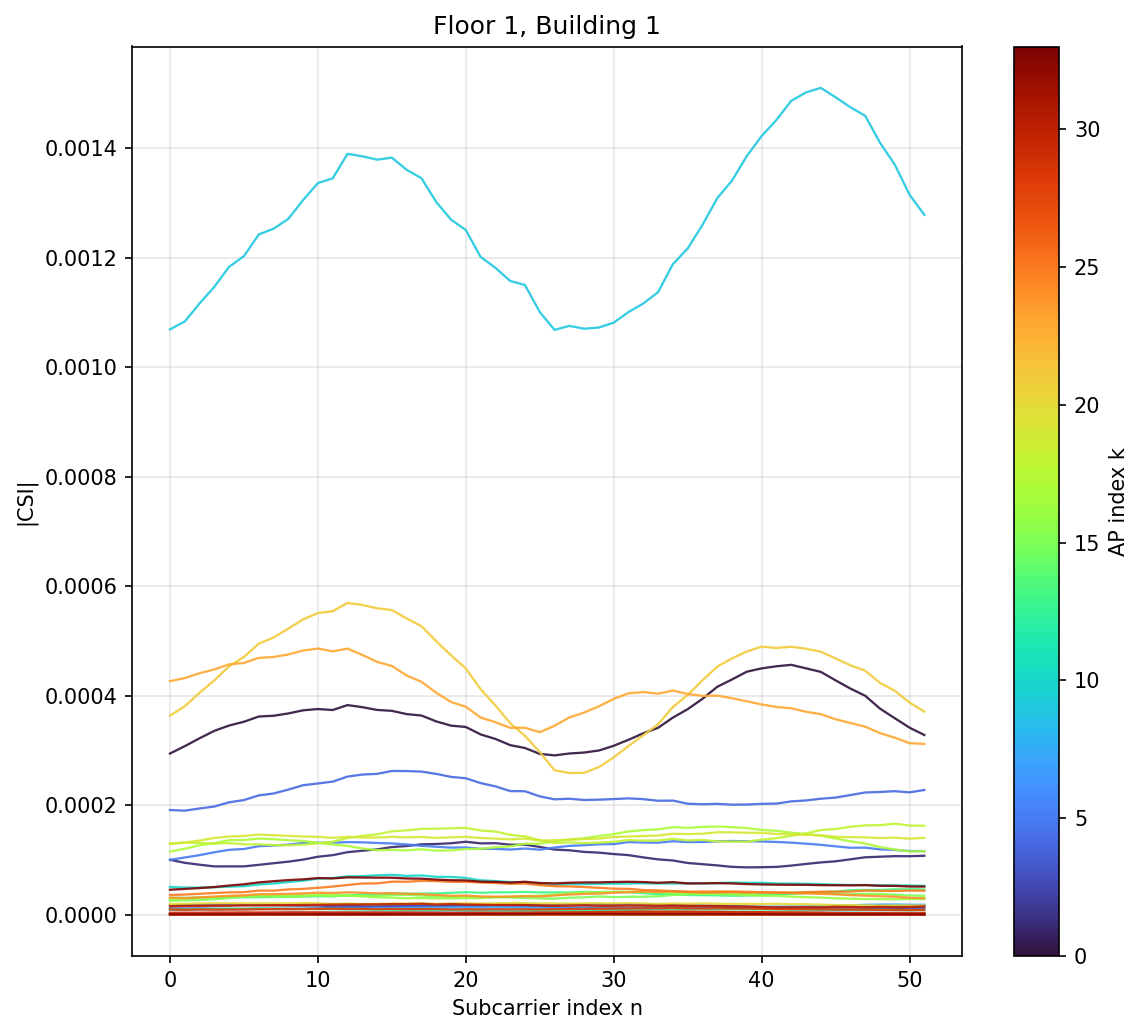}
        \caption{CSI Magnitudes}
        \label{fig:sub2}
    \end{subfigure}
    \caption{Variable AP Length and Power Variations}
    \label{fig:twofigures}
\vspace{-0.4cm}
\end{figure}


In this work, we develop our method based on the following challenges and opportunities. Our extensive measurement campaign spans 16 buildings and collects data from highly diverse environments. Moreover, even within a single environment, such as a floor, the measurements exhibit substantial heterogeneity, including a variable number of observed APs and large variations in received signal strength, as detailed below.

\subsubsection{Diverse Geometry}

Collecting Wi-Fi measurements across 16 geometrically diverse buildings introduces substantial environmental heterogeneity. Differences in floor plans, corridor widths, wall materials, and room density lead to distinct blockage patterns and multipath effects, causing signal observations to vary significantly across buildings. As a result, representations learned in one building may not transfer reliably to another. Similar heterogeneity also exists within a single building, where floor-specific layouts and localized clutter further affect both AP visibility and received signal strength. At the same time, this diversity is a key strength of the dataset: it captures the variability of large-scale campus deployments, enables a rigorous evaluation of robustness and cross-environment generalization, and motivates the design of OmniLoc for accurate and reliable anchor-free indoor localization across diverse indoor settings.

\subsubsection{Variable Length of Observed APs}

Fig. \ref{fig:sub1} reveals a fundamental challenge in Wi-Fi-based localization: the number of APs visible to the UE varies substantially across locations. This leads to variable-length observations and, in turn, a highly irregular input space, since different samples are associated with different sets of detectable APs. As a result, the model cannot assume a fixed or spatially consistent measurement structure. Moreover, AP missingness is often location-dependent rather than random, implying that the visibility pattern itself contains useful spatial cues while also making the learning problem more difficult. Such heterogeneity complicates both feature representation and model generalization, particularly across floors and buildings. This observation motivates the need for a unified input representation, which we detail in Section II. 

\subsubsection{Varying Power of Observed APs}  

The plot of AP power variation (Fig. \ref{fig:sub2}) highlights another central challenge in Wi-Fi localization: even when the same AP is observed, its received signal strength can vary significantly across locations and environments. Such variation arises from distance-dependent path loss, blockage, shadowing, multipath fading, as well as differences in building geometry and material properties. As a result, a fixed physical location may not correspond to a stable power signature, while similar RSSI levels may be observed at distinct locations. This makes power-based features inherently noisy and ambiguous, and thus difficult to use for learning robust location-sensitive representations that generalize across floors and buildings. These observations motivate the design of a robust feature extractor that can inherently account for such variability.

\subsection{Problem Formulation}  

For notational simplicity, we omit the timestamp index $t$ and express the inputs in batched form. The CSI input is denoted by $\mathbf{C} \in \mathbb{R}^{B \times 52 \times N_1}$, the RSSI input by $\mathbf{r} \in \mathbb{R}^{B \times 1 \times N_1}$, and the SINR input by $\mathbf{s} \in \mathbb{R}^{B \times 1 \times N_1}$, where $B$ is the batch size. 

Given the input modalities, including \textit{magnitudes of CSI}\footnote{In large commercial WiFi deployments, we empirically observe that CSI phase measurements are often unstable or effectively random, making them difficult to interpret and prone to introducing substantial errors. Motivated by this observation, we focus on CSI magnitudes
and seek a unified representation that is both robust and expressive.}, RSSI, SINR, and AP presence indicators (denoted by ID1 and ID2), the anchor-free UE localization problem is to learn a neural network that can jointly and reliably represent these measurements and map them to the UE location through a robust nonlinear function. The target output is defined as $(x, y, flr, bld)$, where $(x, y)$ denotes the 2-D coordinates pair; the origin is defined independently in each building, while the X/Y directions are kept consistent across all buildings and the outdoor coordinate system; $flr$ and $bld$ represent the corresponding floor and building indices. The key objective is therefore to learn discriminative representations and an effective optimization strategy that together enable accurate localization across heterogeneous indoor environments.

\begin{figure*}[!t]
	\centering
	\includegraphics[width=0.8\linewidth]{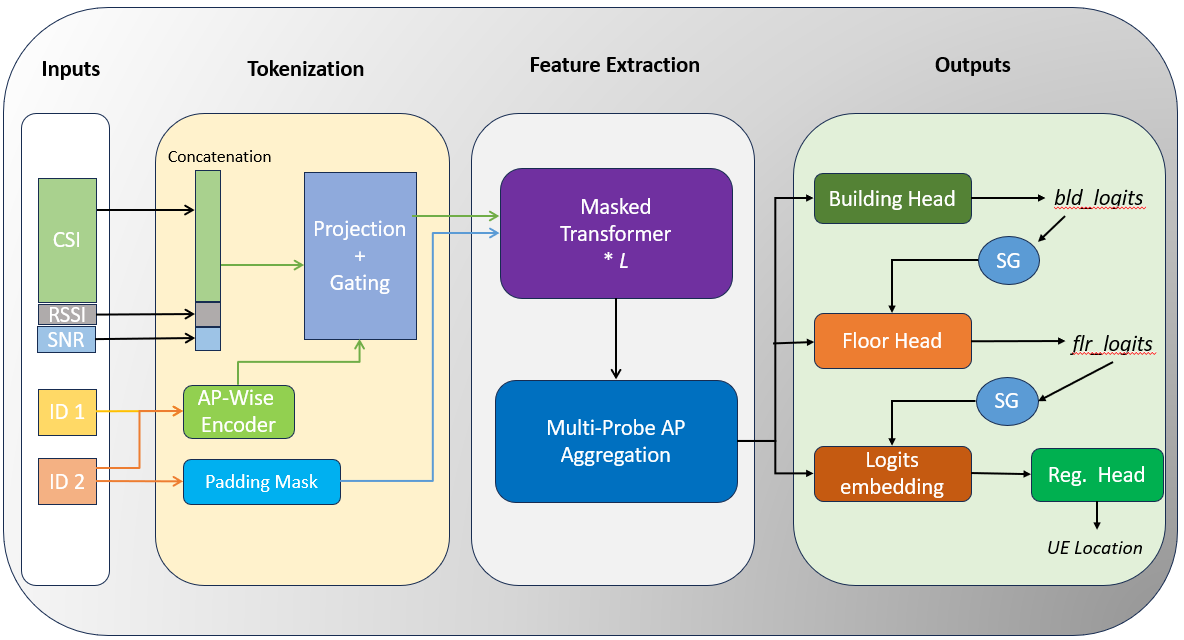} 
	\caption{Overview of our OmniLoc-based Anchor-free UE Localization framework.}
	\label{OmniLoc}
 \vspace{-0.6cm}
\end{figure*}
\section{Proposed Method}

In this section, we present OmniLoc, as illustrated in Fig. \ref{OmniLoc}. OmniLoc consists of three main components: (i) a unified input tokenization module for wireless signal measurements, (ii) a customized Transformer-based feature extractor, and (iii) task-specific output heads.

\subsection{Unified Input Tokenization}

The unified input tokenization (UIT) module takes all available measurements
as input and converts heterogeneous wireless observations, including CSI,
RSSI, SINR, and associated identifiers, into an initial signal representation
through carefully designed embeddings inspired by interpretable sentence
embedding methods~\cite{lin2017structured, reimers2019sentence}.
This design is deliberately simple yet effective: it maps heterogeneous
wireless measurements into a shared representation space, allowing expressive
backbones such as Transformers~\cite{vaswani2017attention} to extract
downstream features more effectively.
The module comprises two stages: input fusion  and
sequence construction.
 
\subsubsection{Input Fusion}

As described in Sec.~\ref{section:system1}, our model takes five types of inputs: CSI, RSSI, SINR, and two AP presence indicators, ${\bf f}$ and ${\bf g}$, which are denoted as ID1 and ID2 in Fig.~\ref{OmniLoc}.
To preserve the AP-wise structure of the measurements, we first concatenate
the magnitudes of CSI ($\mathbf{C} \in \mathbb{R}^{B \times N_1 \times 52}$),
RSSI ($\mathbf{r} \in \mathbb{R}^{B \times N_1 \times 1}$), and SINR
($\mathbf{s} \in \mathbb{R}^{B \times N_1 \times 1}$) along the feature axis
for each AP:
\begin{equation}
  \mathbf{F}
  = \bigl[\,\mathbf{C} \mid \mathbf{r} \mid \mathbf{s}\,\bigr]
  \;\in\; \mathbb{R}^{B \times N_1 \times 54},
\end{equation}
where $\mid$ denotes concatenation along the feature dimension, and the
per-AP slice $\mathbf{F}_{b,i,:} \in \mathbb{R}^{54}$ collects all
measurements for AP~$i$ in sample~$b$.
This layout preserves per-AP information while retaining CSI, RSSI, and
SINR as distinct feature components.
These measurements are complementary: CSI characterizes the multipath
frequency response of the channel over 52 subcarriers, while RSSI and SINR
provide scalar summaries of signal strength and link quality.
By grouping them into a single per-AP feature vector, the projection layer
can learn an effective fusion of heterogeneous wireless measurements \footnote{We note that SINR can be regarded as an attribute reflecting the propagation conditions captured by the CSI. Since CSI is inherently time-varying, its behavior is independent of whether an explicit indicator such as SINR is included. SINR may also be interpreted as aggregated information, for example representing AP-wise variance, and therefore its incorporation does not introduce any fundamental concerns or materially affect the learning process. Its inclusion is consistent with the role of learning in our framework, and we provide extensive case studies, including ablation analyses comparing settings with and without SINR. Furthermore, SINR is not utilized in the pre-trained models designed for cross-dataset or cross-environment generalization. Finally, our primary objective is to demonstrate the advantages of large and diverse datasets, while SINR is employed only within the supervised learning setting.}.


Each per-AP feature vector is then projected to the model dimension $d$:
\begin{equation}
  \mathbf{Z}_{b,i}
  = W_f\,\mathbf{F}_{b,i,:} + n_f
  \;\in\; \mathbb{R}^{d},
\end{equation}
where $W_f \in \mathbb{R}^{d \times 54}$ and $n_f \in \mathbb{R}^{d}$. This linear layer provides a learned, task-specific embedding that maps
the raw measurements into the transformer's latent space.

We further modulate each projected AP token using its corresponding
presence flag ${\bf f}_{b,i} \in \mathbb{R}$:
\begin{equation}
  \mathbf{T}_{b,i}
  = \mathbf{Z}_{b,i} \;\odot\; {\bf f}_{b,i}
  \;\in\; \mathbb{R}^{d},
  \quad i = 1,\ldots,N_1,
\end{equation}
where $\odot$ denotes
element-wise multiplication.
This soft gating retains informative tokens for present APs while
suppressing absent ones without forcing them to exactly zero, thereby
preserving gradient propagation during training.
Specifically, this ensures the projection weights $W_f$ continue to
receive gradient updates from all AP positions, including absent-AP slots,
rather than being blocked by a hard zero as in a naive hard-multiply
formulation.

\subsubsection{Sequence Construction}

The global AP signal vector $\mathbf{g} \in \mathbb{R}^{B \times N_2}$
captures a coarse fingerprint of the entire AP database, independent of
the subset of APs that are locally visible.
To preserve per-AP identity, we encode each scalar entry ${\bf g}_{b,j}$
independently using a shared projection:
\begin{equation}
  e_{b,j}
  = \mathrm{ReLU}\!\left(W_e\,{\bf g}_{b,j} + b_e\right)
  \;\in\; \mathbb{R}^{d},
  \quad j = 1,\ldots,N_2,
\end{equation}
where $W_e \in \mathbb{R}^{d \times 1}$ and $b_e \in \mathbb{R}^{d}$.
Unlike a single $\mathbb{R}^{N_2 \to d}$ projection, this shared
per-element encoder applies the same transformation to every AP entry,
reducing the tendency to overfit to AP index positions.

We then aggregate the $N_2$ encoded vectors by mean pooling:
 
\begin{equation}
  \bar{\mathbf{g}}_b
  = \frac{1}{N_2} \sum_{j=1}^{N_2} e_{b,j}
  \;\in\; \mathbb{R}^{d}.
\end{equation}
 
This yields a permutation-invariant summary of the global AP database,
which is appropriate because $\mathbf{a}_2$ has no natural ordering. To distinguish this summary from the AP tokens in the subsequent sequence, we add a learned positional embedding to form a dedicated GLOBAL token:
\begin{equation}
  \mathbf{g}_b = \bar{\mathbf{g}}_b + \mathbf{p}_0
  \;\in\; \mathbb{R}^{d},
\end{equation}
where $\mathbf{p}_0 \in \mathbb{R}^{d}$ is applied only to the GLOBAL
token at index~0.
AP tokens receive no positional encoding because their ordering within the
sequence is arbitrary; assigning fixed position indices to unordered APs
would introduce spurious structure that could mislead the transformer's
attention mechanism. Finally, we concatenate the GLOBAL token with the $N_1$ gated AP tokens
along the sequence dimension to form the transformer input:
\begin{equation}
\mathbf{X}
=
\Bigl[
\underbrace{\mathbf{g}_b}_{\mathrm{GLOBAL}}
\;\Big|\;
\underbrace{\mathbf{T}_{b,1}\;\cdots\;\mathbf{T}_{b,N_1}}_{\mathrm{AP\ tokens}}
\Bigr]
\in \mathbb{R}^{B \times (1 + N_1) \times d}.
\end{equation}

\subsection{Feature extraction}

Built on top of the tokenization module, the customized Transformer serves as the feature extraction backbone. It captures interactions among heterogeneous wireless signal components and their coupling with geometric information, enabling the network to learn environment-aware representations directly from data.

\subsubsection{Geometry aware Transformer} 

We feed the resulting sequence into a stack of $L$ standard transformer
encoder layers, each comprising multi-head self-attention and a position-wise
feed-forward network, with residual connections and layer normalization. The key-padding mask (KPM), defined as $\mathrm{KPM} \in {0,1}^{B \times L}$, is derived from $\mathbf{f}$ by concatenating a degenerate no-padding column for GLOBAL with the indicator of absent AP slots (ID1). It is applied at every attention layer to exclude padded AP slots from attention computation. Formally, 
\begin{equation}
  \mathbf{X}_{\mathrm{enc}}
  = \mathrm{TransformerEncoder}_{\ell}\!\left(
      \mathbf{X},\;
       \mathrm{KPM}
    \right)
  \;\in\; \mathbb{R}^{B \times L \times d}.
\end{equation}
The encoder preserves the sequence length and hidden dimension, yielding one
$d$-dimensional output vector per token. With masking, attention is restricted to valid AP tokens while padded positions are excluded from participation. This allows the model to focus its capacity on meaningful AP representations without being affected by artificial padding.

\subsubsection{Multi-Probe AP Aggregation}

We consider three downstream tasks—building classification, floor classification, and coordinate regression—which require a fixed-dimensional representation. To this end, we propose Multi-Probe AP Aggregation (MPAA), which compresses the $L$ encoder tokens into a single vector $\mathbf{h}$ via learned multi-head attention pooling with $M$ heads. Unlike mean pooling \cite{lee2019self, gholamalinezhad2020pooling}, which assigns uniform weights to all tokens, MPAA enables each head to learn task-adaptive soft weights over the sequence, selectively emphasizing informative tokens for localization.

We begin by computing the pooling logits, obtained by linearly projecting each token representation into 
$M$ scalars, one for each attention head:
\begin{equation}
  \mathbf{U}_{b}
  = \mathbf{X}_{\mathrm{enc},b}\; W_{\mathrm{pool}}
  \;\in\; \mathbb{R}^{L \times M},
\end{equation}
where $W_{\mathrm{pool}} \in \mathbb{R}^{d \times M}$. Applied to the full
batch, this gives
$\mathbf{U}
  = \mathbf{X}_{\mathrm{enc}}\; W_{\mathrm{pool}}
  \;\in\; \mathbb{R}^{B \times L \times M}.$
Thus, each token receives one pooling logit per head. Compared with query-based cross-attention, this formulation is more lightweight, introducing only the projection matrix $W_{\mathrm{pool}}$ and no separate learned query. We then apply softmax over the sequence dimension $L$ independently for each
head, rather than across heads: 
\begin{equation}
  \mathbf{W}_{b,\,:,\,m} \in\; \mathbb{R}^{L}
  = \mathrm{softmax}_{L}\!\left(\mathbf{U}_{b,\,:,\,m}\right),
   \sum_{l=1}^{L} W_{b,l,m} = 1,
\end{equation}
Thus, each head defines a probability distribution over the $L$ tokens. This
allows every head to perform a differentiable soft selection, ranging from
near one-hot focus on a single token to a more diffuse weighting over multiple
tokens, including the uniform weighting of mean pooling. Each pooling head $m$ computes a weighted combination of the encoded token
representations:
\begin{equation}
  \mathbf{S}_{b,m}
  = \sum_{l=1}^{L} W_{b,l,m}\;\mathbf{X}_{\mathrm{enc},b,l,:}
  \;\in\; \mathbb{R}^{d}.
\end{equation}
This yields one $d$-dimensional summary vector per head and per sample. Stacking
the outputs of all $M$ heads gives $\mathbf{S} \in\; \mathbb{R}^{B \times M \times d}$.
The contraction is performed over the sequence dimension $L$, so each head
aggregates information from the full token set using its learned attention
weights. Different heads can therefore specialize to different aspects of the
environment, such as global database context or salient locally observed APs. 
 
Finally, we concatenate the $M$ head-specific summaries and project the
result back to the model dimension:
\begin{equation}
  \mathbf{h}
  = \mathbf{S} W_{\mathrm{fuse}}^{\top} + b_{\mathrm{fuse}}
  \in \mathbb{R}^{B \times d},
\end{equation}
where $W_{\mathrm{fuse}} \in \mathbb{R}^{d \times Md}$ and
$b_{\mathrm{fuse}} \in \mathbb{R}^{d}$. This fusion step allows the model to
adaptively combine the complementary information captured by different heads
into a single pooled representation.

We note that, in indoor localization, a single strong nearby AP often dominates the location fingerprint. Simple mean pooling can obscure this effect by averaging the dominant AP with weaker or less informative ones. In contrast, MPAA allows one probe to focus on the dominant AP(s) while another aggregates complementary evidence from supporting APs, yielding a more expressive representation. We further provide a theoretical analysis of the proposed MPAA to justify its strong empirical performance. For brevity, the detailed analysis is deferred to Appendix~\ref{sec:appendix}.

\subsection{Geometry-aware Outputs}

Finally, the output heads jointly predict the building index, floor index, and
UE coordinates, yielding a hierarchical localization process: the model first
identifies the building, then refines the prediction to the floor, and finally
estimates the UE coordinates.
To improve coordinate estimation, we introduce a geometry-aware localization
(GLO) module that conditions regression on learned geometric embeddings derived
from the classification outputs.
Specifically, GLO projects the building and floor logits into the model
dimension $d$, fuses them with the aggregated feature $\mathbf{h}$, and feeds
the combined representation to a regression head.
By incorporating discrete spatial context (building and floor identity) as
continuous conditioning signals, the module promotes spatial consistency and
improves the robustness of UE localization.

\subsubsection{Outputs}

\noindent{\bf Building Index.}
Given the pooled representation $\mathbf{h} \in \mathbb{R}^{B \times d}$, we
predict building labels using a two-layer MLP:
\begin{align}
  \mathbf{h}^{(1)}_{\mathrm{bld}}
    &= \mathrm{ReLU}\!\left(
        W_{\mathrm{bld},1}\,\mathbf{h}^{\top} + b_{\mathrm{bld},1}
       \right)
    \;\in\; \mathbb{R}^{d_{\mathrm{ff}} \times B}, \\
  \hat{\mathbf{y}}_{\mathrm{bld}}
    &= W_{\mathrm{bld},2}\,\mathbf{h}^{(1)}_{\mathrm{bld}} + b_{\mathrm{bld},2}
    \;\in\; \mathbb{R}^{C_{\mathrm{bld}} \times B},
\end{align}
where $W_{\mathrm{bld},1} \in \mathbb{R}^{d_{\mathrm{ff}} \times d}$,
$b_{\mathrm{bld},1} \in \mathbb{R}^{d_{\mathrm{ff}}}$,
$W_{\mathrm{bld},2} \in \mathbb{R}^{C_{\mathrm{bld}} \times d_{\mathrm{ff}}}$,
and $b_{\mathrm{bld},2} \in \mathbb{R}^{C_{\mathrm{bld}}}$.
Equivalently, in batch-first form,
$\hat{\mathbf{y}}_{\mathrm{bld}} \in \mathbb{R}^{B \times C_{\mathrm{bld}}}$.
We use a two-layer Multi-Layer Perceptron (MLP) rather than a single linear layer to provide sufficient
nonlinear capacity while keeping the classification head lightweight relative
to the transformer encoder.

\noindent{\bf Floor Index.} Floor identity is building-specific, e.g., floor~3 in building~A is not comparable to floor~3 in building~B. Accordingly, we condition floor prediction on the building prediction. To decouple the two objectives, we apply stop-gradient to $ \hat{\mathbf{y}}_{\mathrm{bld}} $ before feeding it to the floor head. Without this design, the floor loss would back-propagate into the building classifier through an additional path, potentially disrupting optimization of the building cross-entropy loss. The augmented floor input is:
\begin{equation}
  \mathbf{h}_{\mathrm{flr}}
  = \bigl[\,\mathbf{h}\;\|\;
    \mathrm{sg}\!\left(\hat{\mathbf{y}}_{\mathrm{bld}}\right)\bigr]
  \;\in\; \mathbb{R}^{B \times (d + C_{\mathrm{bld}})},
\end{equation}
where $\mathrm{sg}(\cdot)$ denotes stop-gradient (\texttt{.detach()} in
PyTorch) and $\|$ denotes concatenation along the feature dimension.
The floor head maps this augmented input to floor logits:
\begin{align}
  \mathbf{h}^{(1)}_{\mathrm{flr}}
    &= \mathrm{ReLU}\!\left(
        W_{\mathrm{flr},1}\,\mathbf{h}_{\mathrm{flr}}^{\top}
        + b_{\mathrm{flr},1}
       \right)
    \;\in\; \mathbb{R}^{d_{\mathrm{ff}} \times B}, \\
  \hat{\mathbf{y}}_{\mathrm{flr}}
    &= W_{\mathrm{flr},2}\,\mathbf{h}^{(1)}_{\mathrm{flr}} + b_{\mathrm{flr},2}
    \;\in\; \mathbb{R}^{C_{\mathrm{flr}} \times B},
\end{align}
where $W_{\mathrm{flr},1} \in \mathbb{R}^{d_{\mathrm{ff}} \times (d+C_{\mathrm{bld}})}$,
$b_{\mathrm{flr},1} \in \mathbb{R}^{d_{\mathrm{ff}}}$,
$W_{\mathrm{flr},2} \in \mathbb{R}^{C_{\mathrm{flr}} \times d_{\mathrm{ff}}}$,
$b_{\mathrm{flr},2} \in \mathbb{R}^{C_{\mathrm{flr}}}$.
Equivalently, in batch-first form,
$\hat{\mathbf{y}}_{\mathrm{flr}} \in \mathbb{R}^{B \times C_{\mathrm{flr}}}$.
The stop-gradient ensures that the floor loss updates only the floor head, while the building classifier remains supervised solely by its own cross-entropy objective. This avoids conflicting gradient signals, yet still provides the floor head with explicit building-level context.

\noindent{\bf Geometry-Aware UE Coordinate Estimation.} The regression head predicts 2-D coordinates conditioned on both classification outputs. To make this conditioning differentiable, we project the building and floor logits into the model dimension $d$ using two-layer MLPs. Both logit vectors are detached before projection, preventing the regression loss from back-propagating into the classification heads through this conditioning path: 
\begin{align}
  \mathbf{e}_{\mathrm{bld}}
    &= W_{\mathrm{eb},2}\;\mathrm{ReLU}\!\left(
        W_{\mathrm{eb},1}\;\mathrm{sg}\!\left(\hat{\mathbf{y}}_{\mathrm{bld}}\right)
        + b_{\mathrm{eb},1}
       \right) + b_{\mathrm{eb},2}
    \;, \\ 
  \mathbf{e}_{\mathrm{flr}}
    &= W_{\mathrm{ef},2}\;\mathrm{ReLU}\!\left(
        W_{\mathrm{ef},1}\;\mathrm{sg}\!\left(\hat{\mathbf{y}}_{\mathrm{flr}}\right)
        + b_{\mathrm{ef},1}
       \right) + b_{\mathrm{ef},2}
    \;,
\end{align}
where $W_{\mathrm{eb},1} \in \mathbb{R}^{d_{\mathrm{ff}} \times C_{\mathrm{bld}}}$,
$W_{\mathrm{eb},2} \in \mathbb{R}^{d \times d_{\mathrm{ff}}}$, and analogously
for the floor embedding.
Conditioning regression on these embeddings lets the model leverage the discrete structure of indoor localization. Because coordinates are only interpretable within a specific building and floor, the resulting building- and floor-aware context vectors steer the regression head toward geometrically consistent predictions.

The pooled representation and the two geometry embeddings are then
concatenated along the feature dimension:
\begin{equation}
  \mathbf{u}
  = \bigl[\,\mathbf{h}\; \mid \mathbf{e}_{\mathrm{bld}}\; \mid 
    \mathbf{e}_{\mathrm{flr}}\,\bigr]
  \;\in\; \mathbb{R}^{B \times 3d},
\end{equation}
and passed through a two-layer regression MLP to produce the predicted 2-D
coordinates:
\begin{align}
  \mathbf{u}^{(1)}
    &= \mathrm{ReLU}\!\left(
        W_{\mathrm{reg},1}\,\mathbf{u}^{\top} + b_{\mathrm{reg},1}
       \right)
    \;\in\; \mathbb{R}^{d_{\mathrm{ff}} \times B}, \\
  \hat{\mathbf{p}}
    &= W_{\mathrm{reg},2}\,\mathbf{u}^{(1)} + b_{\mathrm{reg},2}
    \;\in\; \mathbb{R}^{2 \times B},
\end{align}
where $W_{\mathrm{reg},1} \in \mathbb{R}^{d_{\mathrm{ff}} \times 3d}$,
$b_{\mathrm{reg},1} \in \mathbb{R}^{d_{\mathrm{ff}}}$,
$W_{\mathrm{reg},2} \in \mathbb{R}^{2 \times d_{\mathrm{ff}}}$,
$b_{\mathrm{reg},2} \in \mathbb{R}^{2}$.
Equivalently, in batch-first form,
$\hat{\mathbf{p}}_b \in \mathbb{R}^2$ is the predicted $(x, y)$ coordinate
for sample $b$.
Concatenating $\mathbf{h}$ alongside the geometry embeddings ensures the
regression head retains direct access to the full encoded scene
representation, rather than depending solely on the (potentially noisy,
early-training) classification logits.

\subsection{Loss Functions and Optimization}

The model is trained end-to-end by minimizing a weighted sum of three losses:
\begin{equation}
\label{loss_func}
  \mathcal{L}
  = \lambda_{\mathrm{bld}}\,\mathcal{L}_{\mathrm{bld}}
  + \lambda_{\mathrm{flr}}\,\mathcal{L}_{\mathrm{flr}}
  + \lambda_{\mathrm{reg}}\,\mathcal{L}_{\mathrm{reg}},
\end{equation}
where $\lambda_{\mathrm{bld}}, \lambda_{\mathrm{flr}}, \lambda_{\mathrm{reg}} > 0$
are task weights.
Both classification tasks use batch-averaged cross-entropy (CE):
\begin{equation}
\small
\mathcal{L}_{\mathrm{bld}}
= \frac{1}{B}\sum_{b=1}^{B}
\mathrm{CE}\!\left(\hat{\mathbf{y}}_{\mathrm{bld},b},\, y^*_{\mathrm{bld},b}\right),
\quad
\mathcal{L}_{\mathrm{flr}}
= \frac{1}{B}\sum_{b=1}^{B}
\mathrm{CE}\!\left(\hat{\mathbf{y}}_{\mathrm{flr},b},\, y^*_{\mathrm{flr},b}\right),
\end{equation} 
where $y^*_{\mathrm{bld},b}$ and $y^*_{\mathrm{flr},b}$ are the ground-truth
building and floor labels for sample $b$.
Lastly, UE's coordinate regression uses mean squared error:
\begin{equation}
  \mathcal{L}_{\mathrm{reg}}
  = \frac{1}{B}\sum_{b=1}^{B}
    \left\|\hat{\mathbf{p}}_b - \mathbf{p}_b\right\|_2^2,
\end{equation}
where $\mathbf{p}_b = (x_b, y_b) \in \mathbb{R}^2$ is the ground-truth
coordinate. 

{\bf Optimization:} We implement the proposed model in PyTorch and train it end-to-end on an NVIDIA GeForce RTX 5090 GPU with mini-batch optimization. Unless stated otherwise, all learnable parameters are randomly initialized and jointly optimized using AdamW. Gradients are back-propagated through the full network, except along explicitly detached paths used to decouple the classification and regression objectives. During training, the Transformer backbone, aggregation module, and prediction heads are optimized jointly under the overall objective.

\section{Case Studies}

We validate OmniLoc\footnote{The code and pretrained model are publicly available \url{https://github.com/Leo-Chu/OmniLoc}} as a competitive universal localization framework by addressing the following questions:
\begin{itemize}
    \item {\bf Q1}: How does OmniLoc compare to the  State of the art (SOTA) localization in our large and diverse dataset?
    \item {\bf Q2}: Do the proposed design components also benefit other localization methods?
    \item {\bf Q3}: Can OmniLoc generalize effectively to unseen environments and adapt flexibly to other datasets in practice?
\end{itemize}

\subsection{Datasets and Performance Metrics}

For evaluation, we mainly use the WiLoc dataset. More details on WiLoc are provided in our recent dataset paper \cite{zhang2026wiloc}, and the dataset is publicly available. In addition, we adopt the WILD dataset from Kaggle
 for cross-environment testing.

\subsubsection{Compared Methods}

Many indoor localization methods have been proposed in the literature. For a fair comparison, we focus only on representative approaches with publicly available code. We adopt their network architectures directly from the official implementations, making only minor modifications where necessary to accommodate our input size. 

\begin{itemize}
    \item {\bf CNN\footnote{\url{https://github.com/sibrendebast/MaMIMO-CSI-positioning-using-CNNs}}:}  A CNN-based indoor localization method \cite{de2020csi} with 13 convolutional layers, similar to the architectures adopted in \cite{hejazi2021dyloc, chu2024exploiting}.  
    \item  {\bf BERT\footnote{ \url{https://github.com/RS2002/CSI-BERT?tab=readme-ov-file}}}: A method employs a BERT-based network \cite{zhao2024finding} to extract informative features from wireless signals for wireless sensing tasks.  
    \item  {\bf LWM\footnote{\url{https://github.com/guangjinpan/LWLM}}:} A large-scale wireless model \cite{pan2025large} is employed for indoor localization based on DeepMIMO dataset.     
 
\end{itemize}

\subsubsection{Performance Metrics}

To evaluate positioning accuracy, we adopt four complementary metrics. For each test sample $i$, we first define the localization error as the Euclidean distance between the estimated position $\hat{\mathbf{p}}_i$ and the ground-truth position $\mathbf{p}_i$:
\begin{equation}
e_i = \left\| \hat{\mathbf{p}}_i - \mathbf{p}_i \right\|_2, \quad i=1,\dots,N.
\end{equation}
Based on $e_i$, the \textit{Mean Localization Error (MLE)} and \textit{Root Mean Squared Error (RMSE)} are defined as
\begin{equation}
\text{MLE} = \frac{1}{N}\sum_{i=1}^{N} e_i,
\qquad
\text{RMSE} = \sqrt{\frac{1}{N}\sum_{i=1}^{N} e_i^2 }.
\end{equation}
Compared with MLE, RMSE penalizes large localization errors more heavily and is therefore more sensitive to outliers. We also report the \textit{Cumulative Distribution Function (CDF)} of the localization error, which characterizes the fraction of test samples whose error is below a given threshold. Based on the CDF, we further use the \textit{90th Percentile Error (P90)}, defined as the error threshold below which 90\% of the localization errors fall.
Together, these metrics provide a comprehensive evaluation of both the central tendency and the tail behavior of the localization error distribution.

\begin{figure}[!t]
	\centering
	\includegraphics[width=0.8\linewidth]{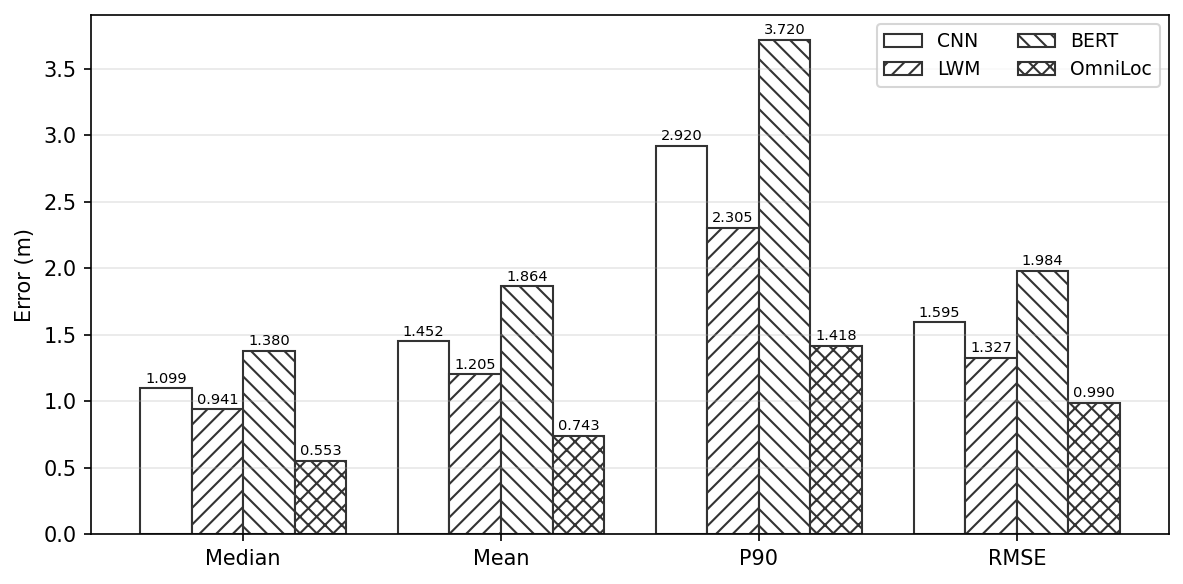} 
	\caption{Qualitative comparisons with SOTA methods.}
	\label{method_comparison_all}
 \vspace{-0.5cm}
\end{figure}
\begin{figure}[!t]
	\centering
	\includegraphics[width=0.9\linewidth]{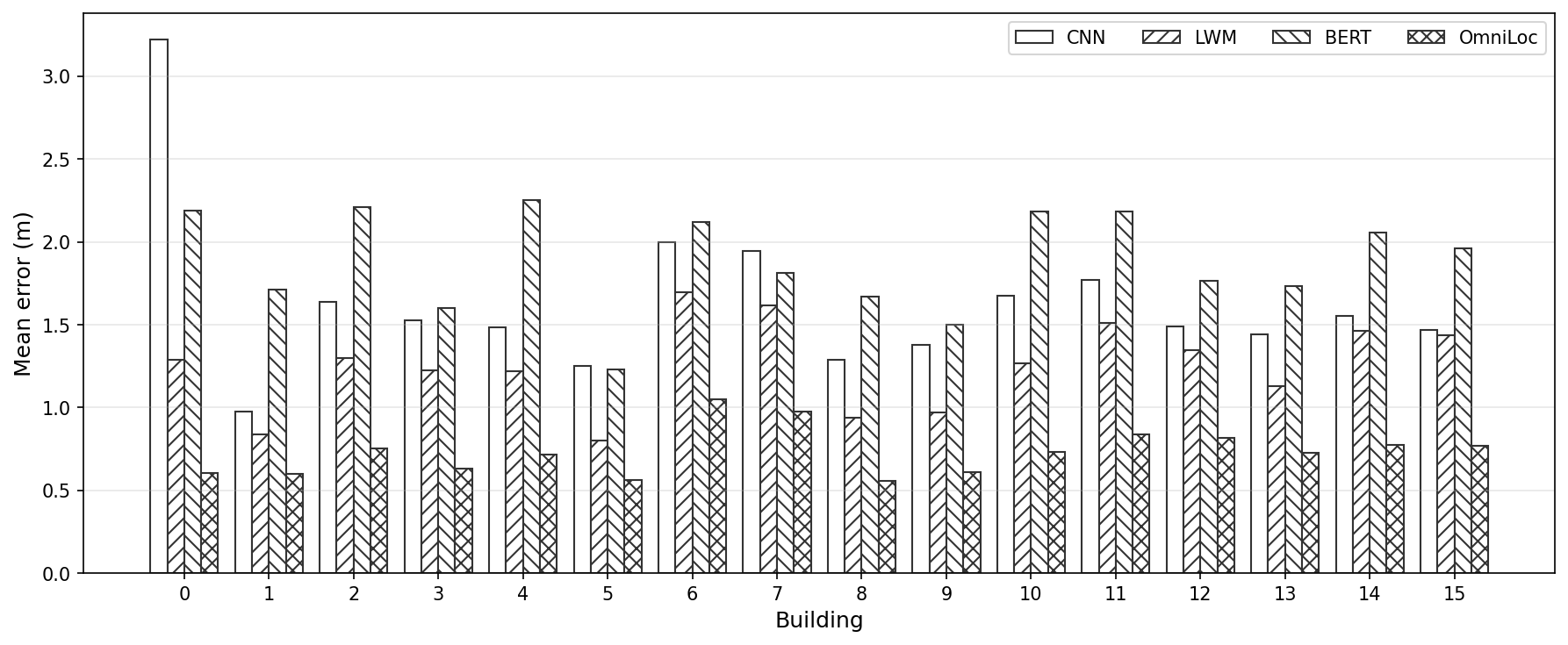} 
	\caption{Building-wise qualitative results. }
	\label{building_wise_plot_no_var}
 \vspace{-0.5cm}
\end{figure}
\begin{figure}[!t]
	\centering
	\includegraphics[width=0.9\linewidth]{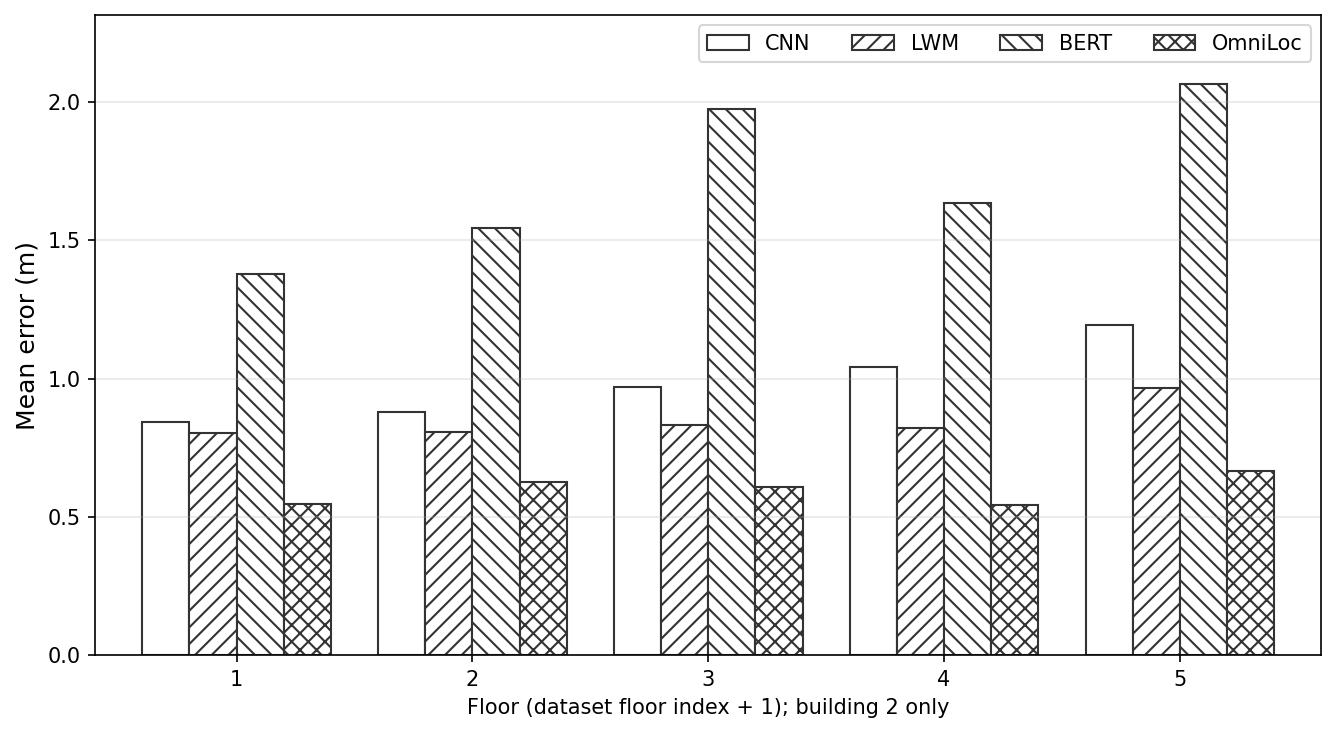} 
	\caption{Floor-wise qualitative results. }
	\label{floor_wise_building1_plot}
 \vspace{-0.9cm}
\end{figure}
\subsubsection{Note on Neural Network Training}

We consider two UE localization setups: {\bf same-environment and cross-environment}. In the same-environment setting, we follow a standard supervised learning protocol, where the data organized in a floor-wise manner and then are randomly split into training, validation, and test sets with a ratio of 7:1:2. In this supervised learning setup, we employ floor-wise stratified splitting. However, due to the high density of the fingerprinting ($< 0.01$ m spacing) and the Wi-Fi beacon interval of $102.4$ ms, residual spatial and temporal correlations likely persist between the training and testing sets, yielding \textit{optimistic} performance.

For cross-environment evaluation, we consider three challenging settings: \textit{Scenario I: leave-one-floor-out}, \textit{Scenario II:  leave-one-building-out}, and \textit{Scenario III:  cross-dataset evaluation}. In \textit{Scenario I}, the training and validation sets are constructed using all samples except those collected on the target floor, and the held-out floor is used exclusively for testing. Unless otherwise specified, we employ the Floor 1 in the Building 1 as the target floor. In \textit{Scenario II}, the training and validation sets are formed using all samples except those collected in the target building, and the held-out building is reserved for testing. Unless otherwise specified, we choose Building 1 (B1) and Building 13 (B13) as the target buildings. In \textit{Scenario III}, we evaluate cross-dataset generalization on the widely used WILD dataset, which contains both RSSI and CSI measurements. In all cross-dataset setups, the neural network training is first conducted using the loss function defined in \eqref{loss_func}. To further study cross-environment adaptation, we investigate parameter-efficient fine-tuning (PEFT) strategies\cite{hu2022lora}, including Low-Rank Adaptation (LoRA) and linear probing (LP). Specifically, LoRA updates only a subset of model parameters, including the query and value projection matrices in the Transformer as well as the regression head, whereas LP fine-tunes only the regression head. It is noted here that we use the zero-shot result as a lower bound and the fully supervised result as an upper bound. Here, \textit{zero-shot} refers to directly applying the pretrained model to data from a new environment without any adaptation, whereas the \textit{fully supervised upper bound} is obtained by training, validation, and testing on data collected from that target environment.


\subsection{Qualitative Comparisons} 

To address {\bf Q1}, we conduct a case study under  the \textit{same-geometry setting} and evaluate localization performance using the four metrics described above. We compare OmniLoc against representative state-of-the-art baselines, including CNN, LVM, and BERT. As shown in Fig. \ref{method_comparison_all}, OmniLoc achieves the lowest error across all four metrics, demonstrating consistently superior localization accuracy. Specifically, OmniLoc attains a median error of 0.553 m, a mean error of 0.743 m, a P90 error of 1.418 m, and an RMSE of 0.990 m, outperforming all competing methods by clear margins. The gains are particularly pronounced on tail-error metrics such as P90 and RMSE, indicating that OmniLoc not only improves average accuracy but also provides stronger robustness in challenging cases. These results validate the effectiveness of OmniLoc and highlight the benefits of its customized input design and network architecture for reliable and precise indoor localization across diverse environments.

We next examine localization performance at a finer spatial granularity by analyzing results at both the building and floor levels. Figures \ref{building_wise_plot_no_var} and \ref{floor_wise_building1_plot} present the building-wise and floor-wise performance of CNN, LVM, BERT, and OmniLoc. While all methods exhibit cross-building variation—reflecting differences in geometry, layout, and propagation conditions—OmniLoc consistently achieves the lowest error in every building. LVM consistently ranks as the second most competitive method in several cases and occasionally outperforms CNN and BERT, indicating its ability to capture certain signal structures under specific deployment conditions. However, its performance is less stable across environments. This variability likely stems from a mismatch between model inductive biases and wireless data characteristics: BERT is primarily optimized for classification tasks, whereas CNNs rely on spatial priors better suited to image-like inputs than to heterogeneous wireless measurements. Overall, the baseline methods remain sensitive to variations in building geometry, AP distribution, and propagation environments, whereas OmniLoc demonstrates stronger robustness across diverse settings due to its geometry-aware design. Fig. \ref{floor_wise_building1_plot} further reports floor-wise MLE across all five floors in Building 1. 


\subsection{Generalization}

\subsubsection{Generalizability of the Proposed Design Components} 

\begin{figure}[!t]
	\centering
	\includegraphics[width=0.8\linewidth]{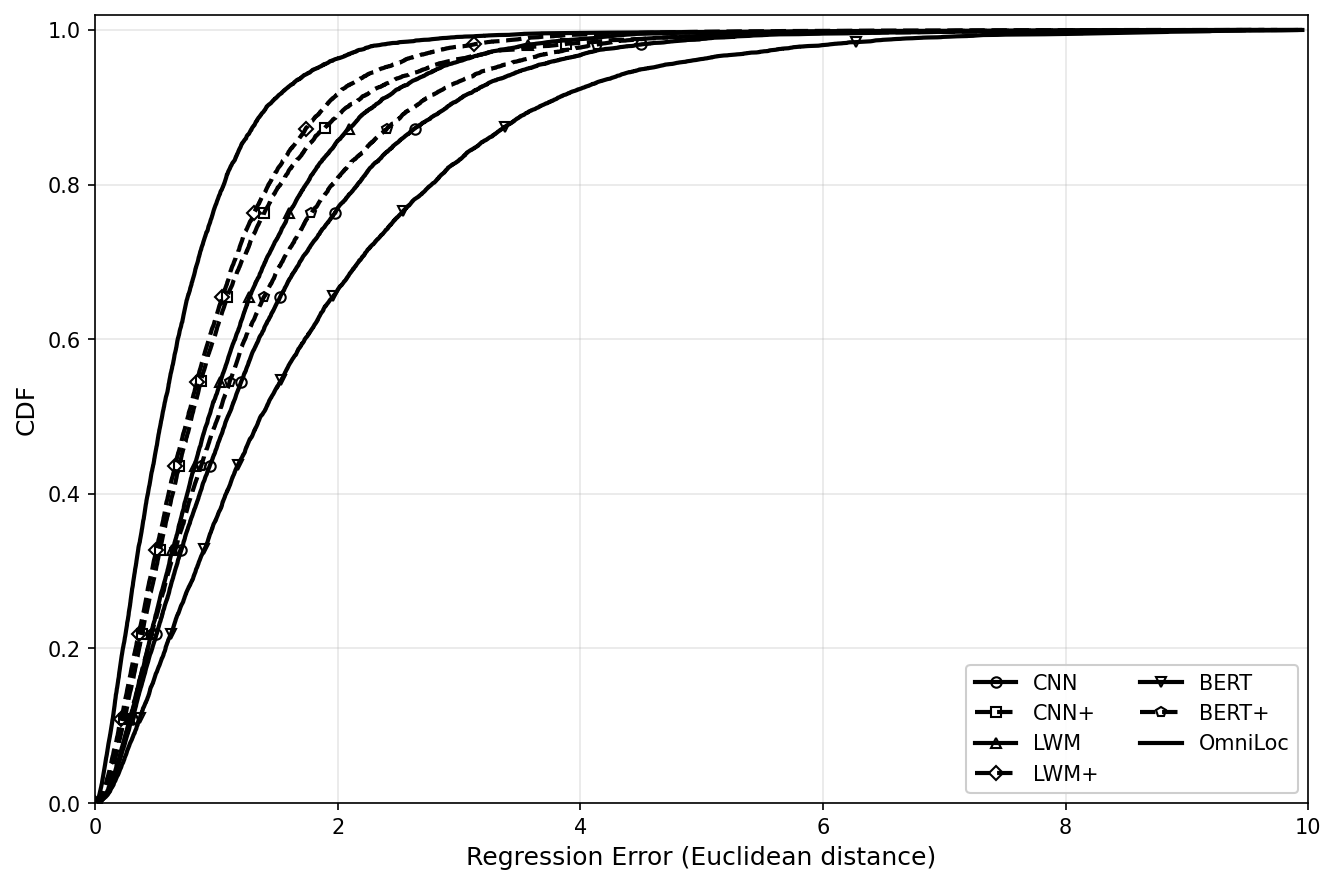} 
	\caption{CDFs of localization errors for all compared methods.}
	\label{cdf_regression_error_case_I}
 \vspace{-0.6cm}
\end{figure}
Fig. \ref{cdf_regression_error_case_I} plots the CDFs of localization error for all methods, including the original baselines and their enhanced “+” variants, which retain the original backbone while incorporating OmniLoc’s design components (unified tokenization, customized feature extraction, and hierarchical prediction). It is shown in Fig. \ref{cdf_regression_error_case_I}  that the enhanced variants consistently shift the corresponding baseline curves upward and leftward, showing that our design yields systematic gains across different backbones. This improvement is especially evident for CNN+, LVM+, and BERT+, each of which outperforms its original version, confirming the general effectiveness and compatibility of the proposed design. Still, OmniLoc achieves the best overall performance, suggesting that fully integrating all proposed design components is more effective than incrementally enhancing existing backbones. These results demonstrate both the superiority of OmniLoc and the transferability of its design principles to other localization models, thereby answering {\bf Q2}.

\subsubsection{Cross-Environment Generalization}  

\begin{figure}[!t]
\centering
\includegraphics[width=0.8\linewidth]{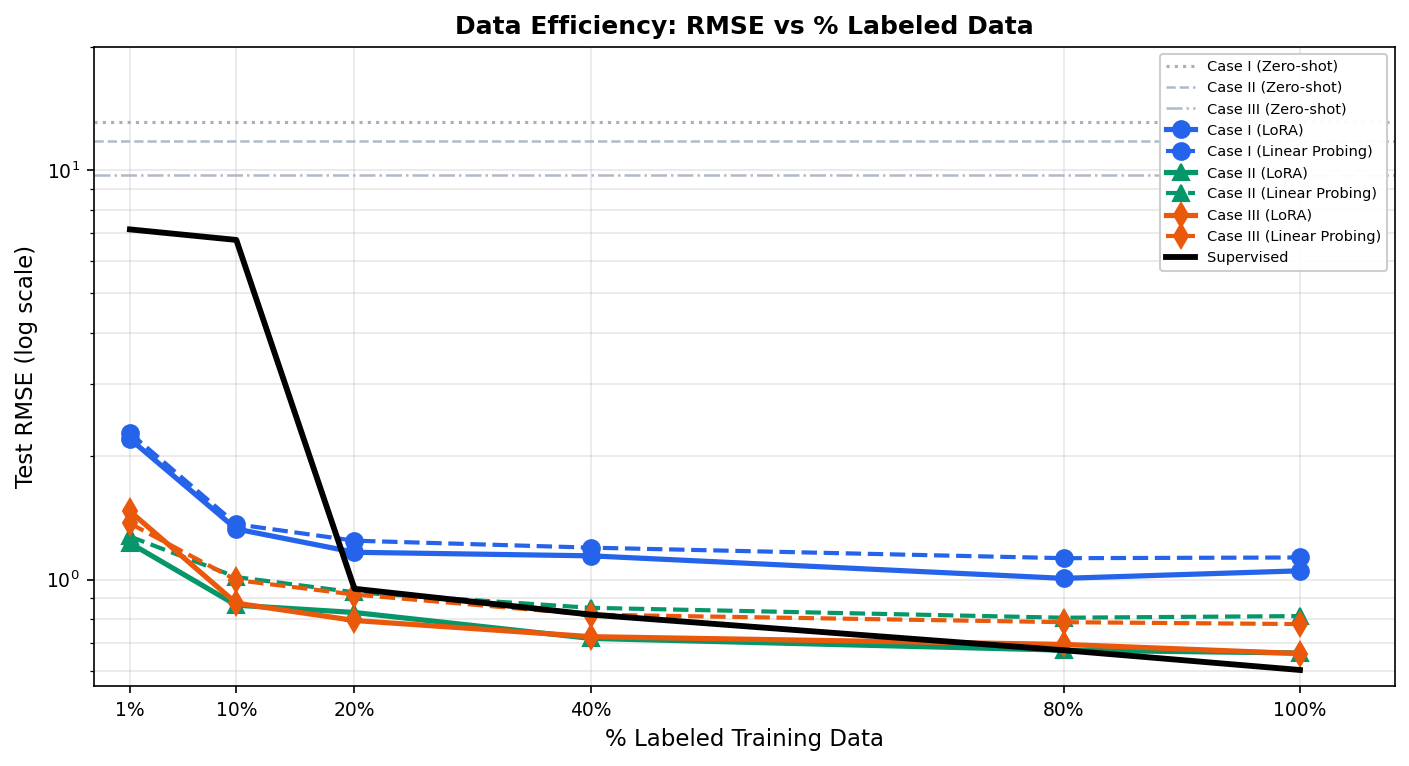} 
\caption{Test RMSE vs. Label Ratio (Scenario I).}
\label{finetune_f_b_studies_data_efficiency}
\vspace{-0.65cm}
\end{figure}

{\bf \textit{ Scenario I: leave-one-floor-out}.} Fig. \ref{finetune_f_b_studies_data_efficiency} evaluates the data efficiency of LoRA and LP across three cross-environment settings in Scenario I. The plot depicts test RMSE against varying fractions of labeled training data, with the \textit{fully supervised upper bound} provided for reference. In all instances, the evaluation and test datasets are kept untouched to maintain the integrity of the performance comparison. We consider three pretrained models with different source-data coverage. In Case I, the model is pretrained using only data from the neighboring floor (Floor 0). In Case II, it is pretrained using data from all other floors in Building 1. In Case III, it is pretrained using all available data except that from Floor 1 in Building 1. For downstream adaptation, we evaluated two strategies: LoRA (represented by solid lines) and LP (represented by dashed lines).  

Several key observations can be drawn from Fig. \ref{finetune_f_b_studies_data_efficiency}. First, adaptation difficulty varies across the three configurations, with Case III    consistently achieving the lowest RMSE overall. Case III tracks the fully supervised upper bound most closely as the labeled fraction increases, demonstrating the strongest transferability. This suggests that pretraining on sufficiently diverse data allows the model to learn highly generalizable representations.  Second, a comparison between adaptation and supervised learning from scratch reveals a distinct \textit{crossover} effect. In the low-label regime (1\%–10\%), PEFT methods provide a substantial performance advantage—maintaining RMSE values nearly an order of magnitude lower than the supervised-from-scratch baseline. This highlights OmniLoc's  robustness in few-shot scenarios, where it successfully leverages its diverse pretraining. Conversely, as the labeled fraction exceeds 20\%, the supervised-from-scratch model improves rapidly, eventually achieving the lowest overall RMSE at 100\% data. This indicates that while OmniLoc offers a superior starting point, the model can still successfully specialize when a large volume of environment-specific data is available. Third, both LoRA and LP remain highly competitive throughout the evaluation. In Case III specifically, LoRA  maintains a slight but consistent edge over LP, further validating the effectiveness of fine-tuning low-rank adapters for cross-environment localization. Overall, Fig. \ref{finetune_f_b_studies_data_efficiency} demonstrates that OmniLoc facilitates effective adaptation and remains robust even when labeled data is extremely scarce.

{\bf \textit{ Scenario II: leave-one-building-out}.} Comparative performance under Scenario II is illustrated in Fig. \ref{finetune_case3_curves}, which plots test RMSE as a function of labeled training data. We compare LoRA and LP against zero-shot and fully supervised baselines across two representative buildings: B1 (highest data volume) and B13 (largest number of floors).

As shown in Fig. \ref{finetune_case3_curves}, both PEFT methods achieve substantial error reductions relative to zero-shot transfer, with the steepest gains occurring between the 1\% and 10\% labeling regimes (low data regime). Notably, in this low data setting, PEFT methods outperform supervised learning from scratch, indicating that even limited supervision, when applied to pretrained representations, enables more effective cross-dataset adaptation than purely supervised approaches. This result further underscores the value of developing large-scale localization models and the necessity of pretraining them on diverse datasets to learn transferable representations.
Across all labeling ratios, LoRA consistently surpasses LP, suggesting that adapting a small subset of internal parameters is more effective than restricting optimization to the output head. This advantage is especially pronounced in the low-label regime, where the ability to refine latent representations is crucial for capturing environment-specific signal characteristics.


Moreover, as the labeled fraction increases, the RMSE of both PEFT methods follows a power-law decay, eventually approaching an asymptotic plateau that reflects diminishing returns from additional annotations. Although the performance gap between PEFT methods and the fully supervised benchmark narrows substantially in data-constrained regimes, a residual gap persists even at higher data fractions (e.g., $\geq$10\% in B1 and $\geq$40\% in B13). This behavior suggests that Transformers trained on heterogeneous localization data may be susceptible to catastrophic forgetting \cite{kirkpatrick2017overcoming}, where the constraints imposed by parameter-efficient updates limit the model’s ability to fully match the performance of end-to-end retraining. Overall, Fig. \ref{finetune_case3_curves} indicates that OmniLoc supports highly sample-efficient adaptation under Scenario II, with LoRA emerging as the most effective strategy for balancing annotation cost and localization accuracy.

\begin{figure}[!t]
	\centering
	\includegraphics[width=0.8\linewidth]{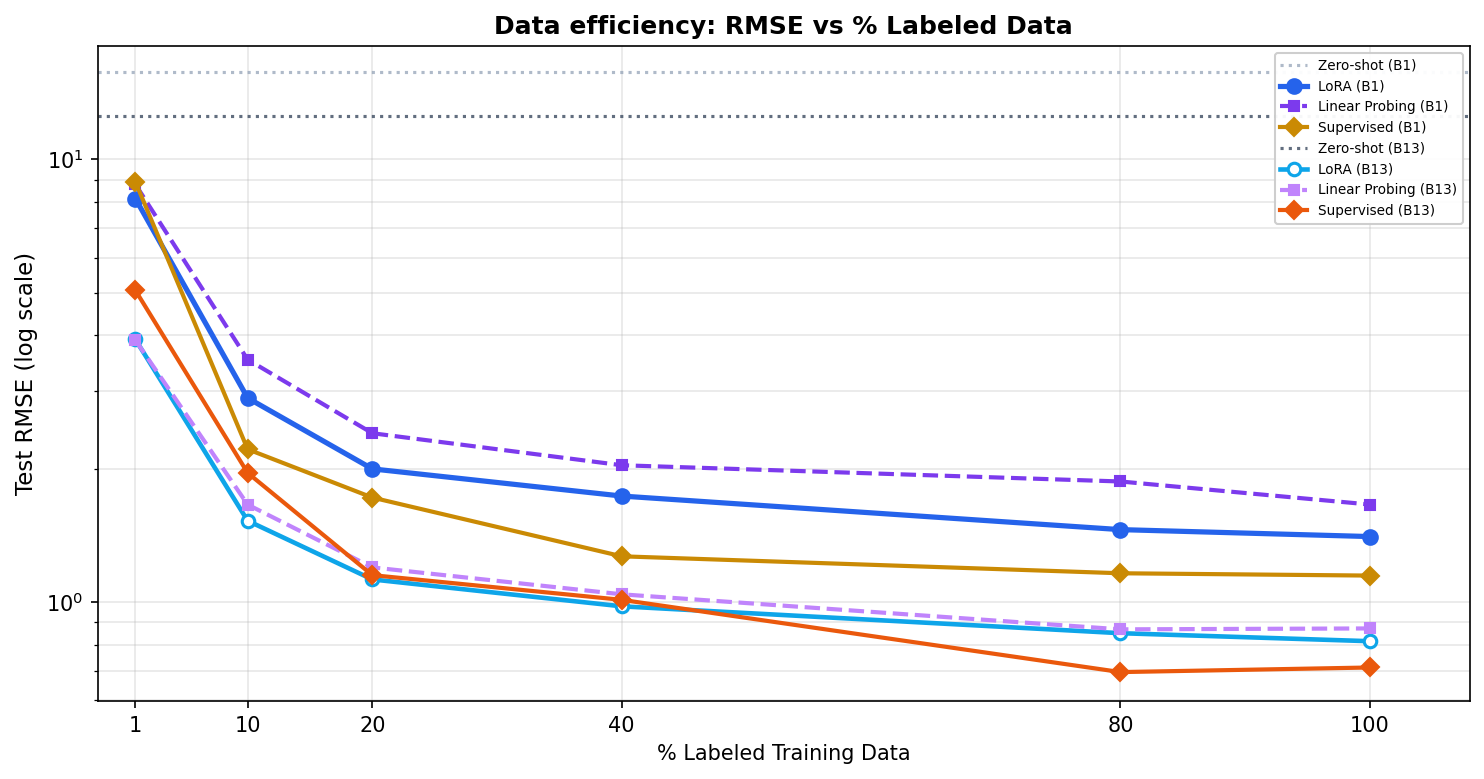} 
	\caption{Test RMSE vs. Label Ratio (Scenario II).}
	\label{finetune_case3_curves}
 \vspace{-0.6cm}
\end{figure}

{\bf \textit{ Scenario III: cross-dataset evaluation.}}
Fig. \ref{finetune_wild_curves} plots RMSE as a function of labeled-data fraction in an even more challenging Scenario III on the WILD dataset, comparing LoRA, LP, and the fully supervised upper bound across two environments, ENV1 and ENV2. Several trends emerge. First, in both environments, localization error decreases sharply as supervision increases from extremely scarce labels to modest labeling ratios, indicating that even limited target-domain annotation yields substantial adaptation gains. Beyond this low-label regime, improvements become more gradual, reflecting diminishing returns from additional labels. One plausible explanation is that, under such challenging conditions, PEFT methods have limited capacity to further refine representations, as only a small subset of model parameters is updated during adaptation.

Second, LoRA consistently outperforms LP across all labeling ratios in both ENV1 and ENV2, with the advantage most pronounced in the low-label regime. This result suggests that effective cross-dataset adaptation requires not only tuning the final predictor, but also updating a small subset of internal model parameters to better align the learned representation with the target-domain signal characteristics. Third, \textit{adaptation difficulty differs across environments}. In ENV1, LoRA achieves low RMSE with limited supervision and remains close to the \textit{fully supervised upper bound} throughout. In ENV2, both methods exhibit higher error, indicating a more severe domain shift; nevertheless, LoRA maintains a clear advantage over LP and steadily approaches the supervised upper bound as the labeled fraction increases.


\begin{figure}[!t]
	\centering
	\includegraphics[width=0.9\linewidth]{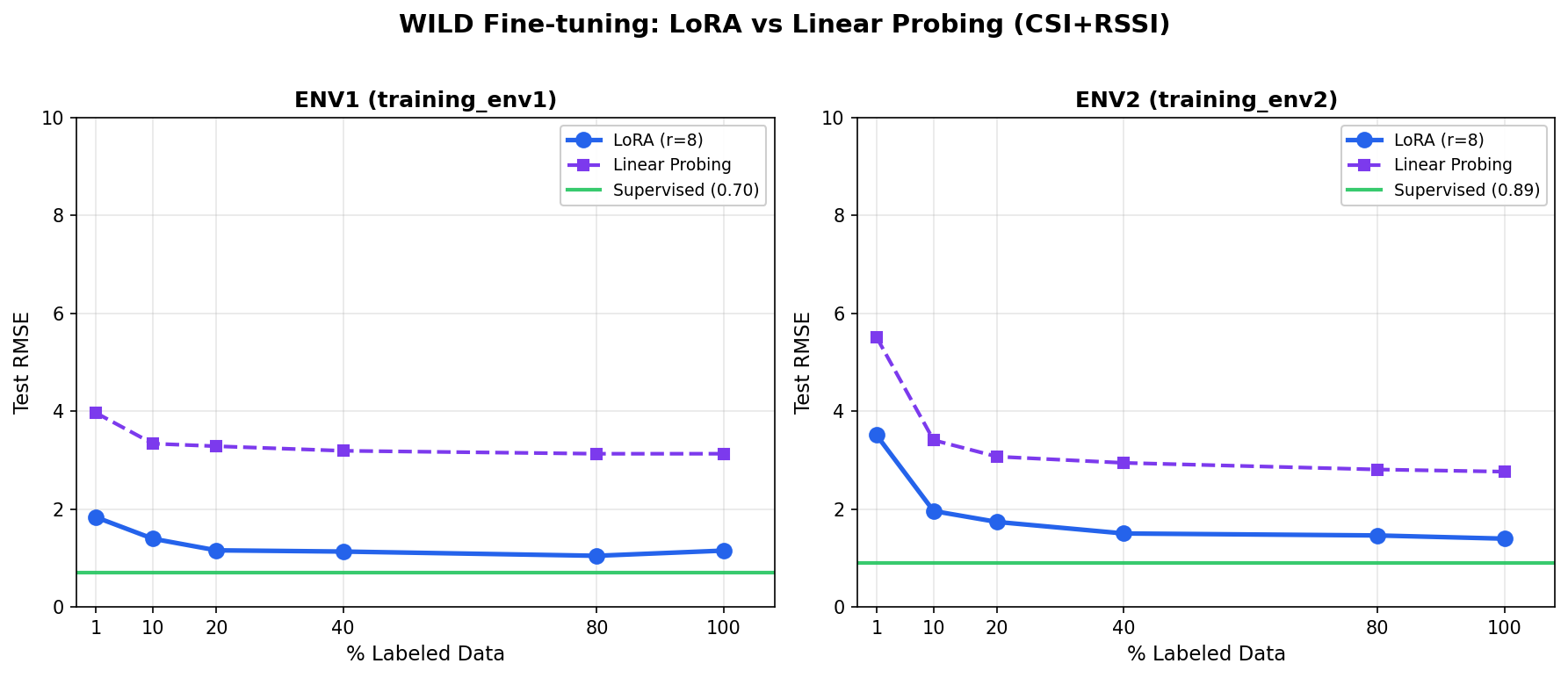} 
	\caption{Test RMSE versus the percentage of labeled data (Scenario III).}
	\label{finetune_wild_curves}
 \vspace{-0.6cm}
\end{figure}


Overall, the results in Figs. 7–10 show that, with appropriate fine-tuning strategies, OmniLoc supports effective and data-efficient cross-dataset adaptation. In particular, LoRA offers a favorable tradeoff between labeling cost and localization accuracy, achieving strong performance with limited supervision while consistently outperforming LP across all environments, providing a positive answer to {\bf Q3}. 

\subsection{Ablation Studies and Complexity}

\subsubsection{Ablation Studies}

\begin{table}[t]
\centering
\caption{Ablation study on different input combinations, network components, and loss terms.} 
\label{tab:ablation}

\scriptsize
\setlength{\tabcolsep}{3pt}
\renewcommand{\arraystretch}{1.05}

\begin{tabular}{c ccc ccc ccc c}
\toprule
\multirow{2}{*}{Case} 
& \multicolumn{3}{c}{Input} 
& \multicolumn{3}{c}{Component} 
& \multicolumn{3}{c}{Loss} 
& RMSE $\downarrow$ \\

\cmidrule(lr){2-4} \cmidrule(lr){5-7} \cmidrule(lr){8-10}

& CSI & RSSI & SINR 
& UIT   & MPAA & GLO 
& $\mathcal{L}_{\mathrm{bld}}$   & $\mathcal{L}_{\mathrm{flr}}$ & $\mathcal{L}_{\mathrm{reg}}$ 
& \\

\midrule
\multicolumn{11}{c}{\textbf{Input Ablation}} \\
\midrule

1 & \cmark & \xmark & \xmark & \cmark & \cmark & \cmark & \cmark & \cmark & \cmark & 1.554  \\
2 & \xmark & \cmark & \xmark & \cmark & \cmark & \cmark & \cmark & \cmark & \cmark & 1.787 \\
3 &  \xmark & \cmark  & \cmark  & \cmark & \cmark & \cmark & \cmark & \cmark & \cmark & 1.691 \\
4 & \cmark & \cmark & \xmark & \cmark & \cmark & \cmark & \cmark & \cmark & \cmark & 1.545 \\
5 & \cmark & \xmark & \cmark  & \cmark & \cmark & \cmark & \cmark & \cmark & \cmark & 1.243 \\

\midrule
\multicolumn{11}{c}{\textbf{Component Ablation}} \\
\midrule

6  & \cmark & \cmark & \cmark & \xmark & \cmark & \cmark & \cmark & \cmark & \cmark & 1.366 \\
7  & \cmark & \cmark & \cmark & \cmark & \xmark & \cmark & \cmark & \cmark & \cmark & 1.463 \\
8  & \cmark & \cmark & \cmark & \cmark & \cmark & \xmark & \cmark & \cmark & \cmark & 1.325 \\

\midrule
\multicolumn{11}{c}{\textbf{Loss Ablation}} \\
\midrule

9 & \cmark & \cmark & \cmark & \cmark & \cmark & \cmark & \xmark & \xmark & \cmark & 1.374 \\
10 & \cmark & \cmark & \cmark & \cmark & \cmark & \cmark & \cmark & \xmark & \cmark & 1.291 \\
11 & \cmark & \cmark & \cmark & \cmark & \cmark & \cmark & \xmark & \cmark & \cmark & 1.245 \\
\midrule
\multicolumn{11}{c}{\textbf{Full Version}} \\
\midrule
12 & \cmark & \cmark & \cmark & \cmark & \cmark & \cmark & \cmark & \cmark & \cmark & \textbf{0.990} \\
\bottomrule
\end{tabular}
\vspace{-0.6cm}
\end{table}

Tab. \ref{tab:ablation} reports an ablation study on input modalities, architectural components, and loss design. For the input settings, CSI alone achieves lower error than RSSI alone, reducing RMSE from 1.787 to 1.554 (13.0\%). This is reasonable because CSI retains fine-grained spatial and frequency-domain information across subcarriers, whereas RSSI is only a coarse aggregate measurement. Adding SINR further improves both settings, but the gain is larger when CSI is used: RMSE decreases from 1.554 to 1.359 (12.5\%) for CSI, compared with a reduction from 1.787 to 1.662 (7.0\%) for RSSI. This pattern suggests that SINR mainly acts as a quality indicator that complements feature-rich measurements more effectively than coarse ones. In other words, when the input already contains detailed signal structure, SINR helps the model judge how reliable that structure is; when the input is limited to RSSI, the benefit is necessarily smaller. Combining CSI and RSSI yields only a marginal improvement over CSI alone, indicating that RSSI contributes little additional information beyond what is already encoded in CSI, likely because RSSI can be viewed as a heavily compressed summary of the same underlying signal.

The component ablation further shows that every module contributes to the final performance. Removing UIT increases RMSE to 1.366, which is 16.5\% higher than that of the full model. Removing MPAA increases RMSE to 1.463, corresponding to a 24.7\% degradation, while removing GLO increases RMSE to 1.325, or a 13.0\% degradation. These results indicate that all three modules are beneficial, with MPAA contributing the largest performance gain.

For the loss-function ablation, every partial loss variant performs worse than the full objective. Relative to the complete model, RMSE increases to 1.374, 1.291, and 1.245, corresponding to degradations of 17.1\%, 10.1\%, and 6.1\%, respectively. This consistent trend indicates that each loss term contributes complementary supervision, and removing any part weakens the training signal. As a result, the full model, which jointly incorporates all inputs, components, and loss terms, achieves the best overall performance with the lowest RMSE of 0.990. Overall, these results confirm that the effectiveness of the framework comes not from any single design choice alone, but from the synergy among its components.

\subsubsection{Complexity Analysis}


\begin{table}[t]
\centering
\caption{Model efficiency comparison. Params: total parameters; Train: trainable parameters; FLOPs: floating point operations; Inf: inference latency.   } 
\label{tab:model_efficiency}
\begin{tabular}{lcccccc}
\toprule
Model & Params (MB) & Train (MB) & FLOPs (G) & Inf (ms) 
\\
\midrule
CNN & 4.043 & 4.043 & 23.34 & 4.71 
\\
LWM & 4.456 & 4.456 & 88.55 & 6.29 
\\
BERT & 4.049 & 4.049 & 267.32 & 17.49 
\\
OmniLoc & 4.456 & 4.456 & 88.55 & 6.49 
\\
LoRA & 4.505 & 0.997 & 85.18 & 5.93 
\\
LP & 4.456 & 0.948 & 75.08 & 5.31 
\\
\bottomrule
\end{tabular}
\vspace{-0.5cm}
\end{table}
 
Tab. \ref{tab:model_efficiency} compares the efficiency of different localization models in terms of total parameters, trainable parameters, FLOPs, and inference latency. Among the baselines, CNN is the most lightweight, requiring only 23.34 G FLOPs and 4.71 ms latency, whereas BERT incurs the highest computational cost at 267.32 G FLOPs and 17.49 ms. LVM and OmniLoc have the same model size and nearly identical complexity, each with 4.456 MB parameters and 88.55 G FLOPs, indicating that OmniLoc’s accuracy gains do not rely on increased model scale.

PEFT methods further improve OmniLoc’s efficiency. OmniLoc (LoRA) reduces the number of trainable parameters from 4.456 MB to 0.997 MB while keeping FLOPs nearly unchanged and slightly lowering inference latency. OmniLoc (LP) is even more efficient, requiring only 0.948 MB trainable parameters, 75.08 G FLOPs, and 5.31 ms latency. These results show that OmniLoc not only achieves strong localization performance, but also supports efficient adaptation through lightweight fine-tuning, making it well suited for deployment in resource-constrained environments. 






\section{Conclusions}

In this paper, we presented OmniLoc, an environment-interactive foundation model for anchor-free indoor localization from heterogeneous wireless measurements. To the best of our knowledge, OmniLoc is the first foundation-model-based framework tailored to this task. By unifying diverse wireless signals with a shared tokenization scheme, modeling AP-aware geometric structure with a geometry-aware Transformer, and enforcing geometric consistency through a dedicated estimation head, OmniLoc achieves strong localization performance across diverse indoor environments. Results on both a large-scale in-house dataset and a public benchmark show that OmniLoc consistently outperforms state-of-the-art baselines. Beyond the full framework, its key design components also transfer effectively to other localization backbones, suggesting that the gains arise from general and reusable principles. OmniLoc further shows strong robustness under cross-environment and cross-dataset evaluation, demonstrating good generalization to unseen deployments. Overall, these results highlight OmniLoc as a unified, accurate, and extensible solution for wireless indoor localization, and point to the promise of foundation-model design for anchor-free localization with widely available WiFi networks.

\textbf{Limitations:} Our current design has several limitations. First, it relies only on CSI magnitude and therefore does not exploit phase information, which may contain additional localization cues despite being noisy and difficult to calibrate in practice. Better phase error modeling and correction \cite{xie2015precise, ibrahim2018verification, tong2026error} could further improve performance. Second, OmniLoc currently uses only coarse-grained geometric priors, such as floor and building indices, and does not yet incorporate fine-grained environmental context. Extending the framework with complementary modalities, such as LiDAR, IMU, or cameras, is a promising direction for enriching geometric awareness and further improving localization accuracy, potentially toward millimeter-level precision.

\appendices

\section{THEORETICAL RESULTS and Proofs}
\label{sec:appendix}

Throughout this section, the transformer encoder output is
$\mathbf{X} \equiv \mathbf{X}_{\mathrm{enc}} \in \mathbb{R}^{L \times d}$
(batch index suppressed), where $L = 1 + N_1$ and $N_1$ is the number of
locally observed APs.
The scoring matrix is $W_p \in \mathbb{R}^{d \times M}$ with columns
$w_m \in \mathbb{R}^d$ ($m = 1,\ldots,M$).
The attention weights are
$\mathbf{W} = \mathrm{softmax}_L(\mathbf{X}W_p) \in \mathbb{R}^{L \times M}$,
where $\mathrm{softmax}_L$ normalizes over the sequence dimension.
The $m$-th probe summary is
$\mathbf{S}_m = \mathbf{X}^\top\mathbf{W}_{:,m} \in \mathbb{R}^d$, and the
final pooled vector is
$\mathbf{h} = W_f\,\mathrm{vec}(\mathbf{S}) + b_f \in \mathbb{R}^d$,
where $W_f \in \mathbb{R}^{d \times Md}$ fuses all probe summaries.
The following analysis builds on the Deep Sets theorem~\cite{zaheer2017deep},
which characterizes all permutation-invariant set functions.

\begin{theorem}[Deep Sets \cite{zaheer2017deep}]
\label{thm:deepsets}
A function $f : \mathcal{X}^{(L)} \to \mathbb{R}$ on finite multisets
is permutation-invariant if and only if it decomposes as:
\begin{equation}
  f\!\left(\{x_1,\ldots,x_L\}\right)
  = \rho\!\left(\sum_{l=1}^{L} \phi(x_l)\right),
  \label{eq:deepsets}
\end{equation}
for suitable $\phi : \mathcal{X} \to \mathbb{R}^K$ and
$\rho : \mathbb{R}^K \to \mathbb{R}$.
\end{theorem}

\begin{proof}[Proof of Theorem~\ref{thm:deepsets} (sketch)]
\textit{Sufficiency ($\Leftarrow$):}
Any function of the form $\rho(\sum_l \phi(x_l))$ is permutation-invariant
because summation commutes with reordering.

\textit{Necessity ($\Rightarrow$):}
For a countable $\mathcal{X}$, one can construct an injective $\phi$ that
maps each element to a distinct basis vector, making the sum a count vector
that encodes multiset membership.
The function $\rho$ can then recover any permutation-invariant $f$ from this
representation.
See~\cite{zaheer2017deep} for the full construction.
\end{proof}

\begin{remark}[MPAA as a Deep Sets instance]
\label{rem:mpaa-ds}
MPAA is a concrete learnable instantiation of~\eqref{eq:deepsets}:
the transformer encoder plays the role of $\phi$, with $\phi(x_l) = \mathbf{X}_{l,:}$,
and the probe-weighted aggregation and linear fusion $W_f\,\mathrm{vec}(\mathbf{S}) + b_f$
implement $\rho$.
\end{remark}

\begin{proposition}[Mean pooling as a special case]
\label{prop:mean}
When $W_p = \mathbf{0}_{d \times M}$, every probe of MPAA reduces to
mean pooling: $\mathbf{S}_m = \bar{\mathbf{X}}$ for all $m$.
\end{proposition}

\begin{proof}
$\mathbf{U}_{:,m} = \mathbf{X}\,\mathbf{0} = \mathbf{0}_L$.
By symmetry, $\mathrm{softmax}(\mathbf{0}_L)_l = e^0 / (L\,e^0) = 1/L$
for all $l$.
Therefore $\mathbf{S}_m = \sum_l (1/L)\,\mathbf{X}_{l,:} = \bar{\mathbf{X}}$.
\end{proof}
 
Proposition~\ref{prop:mean} establishes that mean pooling is always
reachable by MPAA (by setting $W_p = \mathbf{0}$), so
$\mathcal{F}_{\mathrm{mean}} \subseteq \mathcal{F}_M$.
The following lemma shows the inclusion is strict: there exist functions
computable by MPAA that no mean-pooling MLP can represent.

\begin{lemma}[Strict enrichment over mean pooling]
\label{lem:strict}
Let $\mathcal{F}_{\mathrm{mean}}$ be the function class of mean pooling
followed by an MLP, and $\mathcal{F}_M$ the class of $M$-head MPAA
followed by $W_f$.
For $M \geq 2$:
\begin{equation}
  \mathcal{F}_{\mathrm{mean}} \subsetneq \mathcal{F}_M.
\end{equation}
\end{lemma}

\begin{proof}
\textit{Containment} ($\mathcal{F}_{\mathrm{mean}} \subseteq \mathcal{F}_M$):
follows directly from Proposition~\ref{prop:mean}.

\textit{Strictness} ($\mathcal{F}_{\mathrm{mean}} \subsetneq \mathcal{F}_M$):
We exhibit a function in $\mathcal{F}_M \setminus \mathcal{F}_{\mathrm{mean}}$.
Let $g(\mathbf{X}) = \max_{l}\,\mathbf{X}_{l,1}$
(the maximum first-coordinate value across all tokens).

\textit{(a) $g \notin \mathcal{F}_{\mathrm{mean}}$:}
Mean pooling computes $\bar{\mathbf{X}} = \frac{1}{L}\sum_l \mathbf{X}_{l,:}$,
which depends on $\mathbf{X}_{l,1}$ only through the average
$\bar{X}_1 = \frac{1}{L}\sum_l X_{l,1}$.
Since the maximum and the mean of a sequence can differ arbitrarily,
no MLP applied to $\bar{\mathbf{X}}$ can compute $g$ exactly.

\textit{(b) $g \in \mathcal{F}_M$ (approximately):}
Set $w_1 = \alpha\,\mathbf{e}_1$ for head $m = 1$, where $\mathbf{e}_1$
is the first standard basis vector and $\alpha > 0$ is a temperature
parameter; the remaining heads are unconstrained.
The attention weight for token $l$ in head~1 is:
\begin{equation}
  W_{l,1}(\alpha)
  = \frac{e^{\alpha X_{l,1}}}{\sum_{l'} e^{\alpha X_{l',1}}}.
\end{equation}
As $\alpha \to \infty$, the softmax concentrates on the token with the
largest first coordinate:
\begin{equation}
  \mathbf{S}_{1} \;\xrightarrow{\;\alpha \to \infty\;}
  \mathbf{X}_{\hat{l},:},
  \qquad \hat{l} = \arg\max_{l}\, X_{l,1}.
\end{equation}
The first component of $\mathbf{S}_1$ therefore converges to
$\max_l X_{l,1} = g(\mathbf{X})$, which a linear $W_f$ can read off
exactly.
Since $\alpha$ is a learnable parameter, the optimizer can approximate
$g$ to arbitrary precision, so $g \in \mathcal{F}_M$.
\end{proof}




\ifCLASSOPTIONcaptionsoff
  \newpage
\fi



%
 
\bibliographystyle{IEEEtran}
\bibliography{reference}

%








\end{document}